\documentclass{article}

\usepackage{PRIMEarxiv}

\usepackage[utf8]{inputenc} % allow utf-8 input
\usepackage[T1]{fontenc}    % use 8-bit T1 fonts
\usepackage{hyperref}       % hyperlinks
\usepackage{url}            % simple URL typesetting
\usepackage{booktabs}       % professional-quality tables
\usepackage{amsthm,amsmath, amssymb, latexsym,bbm}       % blackboard math symbols
\usepackage{nicefrac}       % compact symbols for 1/2, etc.
\usepackage{microtype}      % microtypography
\usepackage{lipsum}
\usepackage[numbers]{natbib}
\usepackage{fancyhdr}       % header
\usepackage{graphicx,subfig}       % graphics
\graphicspath{{media/}}     % organize your images and other figures under media/ folder

\newtheorem{theorem}{Theorem}

\newtheorem{definition}{Definition}

\newtheorem{remark}{Remark}[section]

\def \R{{\mathbb R}}

\newcommand{\mbR}{\mathbbm R}
\newcommand{\mbN}{\mathbbm N}
\DeclareMathOperator{\argmin}{argmin}
\def\SPD{\text{SPD}}

\title{Intrinsic and extrinsic deep learning on manifolds }

\author{
  Yihao Fang \\
  Department of Applied and Computational Mathematics and Statistics \\
  University of Notre Dame \\
  \texttt{yfang5@nd.edu} \\
   \And
  Ilsang Ohn \\
  Department of Statistics \\
  Inha University \\
  \texttt{ilsang.ohn@inha.ac.kr} \\
     \And
  Vijay Gupta  \\
  School of Electrical and Computer Engineering \\
  Purdue University \\
  \texttt{gupta869@purdue.edu} \\
     \And
  Lizhen Lin \\
  Department of Applied and Computational Mathematics and Statistics \\
  University of Notre Dame\\
  \texttt{lizhen.lin@nd.edu} \\
}

\begin{document}
\maketitle

\begin{abstract}
We propose extrinsic and intrinsic deep neural network architectures as general frameworks for deep learning on manifolds. Specifically, extrinsic deep neural networks (eDNNs) preserve geometric features on manifolds by utilizing an equivariant embedding from the manifold to its image in the Euclidean space. Moreover, intrinsic deep neural networks (iDNNs) incorporate the underlying intrinsic geometry of manifolds via exponential and log maps with respect to a Riemannian structure. Consequently, we prove that the empirical risk of the empirical risk minimizers (ERM) of eDNNs and iDNNs converge in optimal rates. Overall, The eDNNs framework is simple and easy to compute, while the iDNNs framework is accurate and fast converging. To demonstrate the utilities of our framework, various simulation studies, and real data analyses are presented with eDNNs and iDNNs.
\end{abstract}

% keywords can be removed

\section{Introduction}
The last two decades have witnessed an explosive development in deep learning approaches. These approaches have achieved  breakthrough performance in  a broad range of  learning problems from a variety of  applications fields  such as imaging recognition \citep{krizhevsky2012imagenet}, speech recognition \citep{speechhinton}, natural language  processing \citep{Bahdanau2015NeuralMT}   and other areas  of computer vision \citep{Voulodimos2018DeepLF}. Deep learning has also served as the main impetus for the   advancement of recent artificial intelligence (AI) technologies. This unprecedented success has been made possible due to the increasing computational prowess,  availability of large data sets, and the development of efficient  computational algorithms for training deep neural networks. There have been increasing efforts to understand the theoretical  foundations of deep neural networks, including in the statistics community \citep{schmidt-hieber2020, ohn2019smooth, kim2018fast, bauer2019deep, suzuki2018adaptivity, Kohler2019OnTR, Fan2019ASO}.

 Most of these efforts from model and algorithmic development to theoretical understanding, however, have been largely focused on the Euclidean domains.  In a  wide range of problems arising in computer and machine vision,  medical imaging,  network science,  recommender systems,   computer graphics, and so on, one often encounters learning problems concerned with non-Euclidean data, particularly manifold-valued data.  For example, in neuroscience,  data collected in diffusion tensor imaging (DTI), now a powerful tool in neuroimaging for clinical trials,  are  represented by the diffusion matrices,  which are $3\times3$ \emph{positive definite matrices} \citep{dti-ref}.    In engineering and machine learning, pictures or images are often  preprocessed or reduced to a collection of \emph{subspaces} with each data point (an image) in the sample data represented by a subspace \citep{subspacepaper, facialsubpace}.  In machine vision, a digital image can also be represented by a set of $k$-landmarks,  the collection of which form  \emph{landmark-based shape spaces}  \citep{kendall84}. One may also encounter data that are stored as \emph{orthonormal frames} \citep{vecdata}, \emph{surfaces}, \emph{curves}, and \emph{networks} \citep{kolaczyk2020}. The underlying space where these general objects belong falls in the general category of \emph{manifolds} whose geometry is  generally well-characterized,  which should be utilized and incorporated for learning and inference. Thus, there is a natural need and motivation for developing deep neural network models over manifolds. 

 This work aims to develop  general deep neural network architectures on manifolds and take some  steps toward understanding their theoretical foundations. The key challenge lies in incorporating the underlying geometry and structure of manifolds in designing deep neural networks. % over manifolds.  
Although some recent works propose  deep neural networks for specific manifolds \citep{Zhang2018GrassmannianLE, matrixback-prop, deep-LG, deepgra},  there is a lack of general frameworks or paradigms that work for arbitrary manifolds. In addition, the theoretical understanding of deep neural networks on manifolds remains largely unexplored. To fill in these  gaps, in this work, we make the following contributions: (1) we develop \emph{extrinsic  deep neural networks (eDNNs)} on manifolds to generalize the popular feedforward networks in the Euclidean space to manifolds via equivariant embeddings. The extrinsic framework is conceptually simple and computationally easy and works for general manifolds where nice embeddings such as \emph{emquivariant embeddings} are available;  
(2) we develop \emph{intrinsic deep neural networks (iDNNs)} for deep learning networks on manifolds employing a Riemannian structure of the manifold; (3) we study theoretical properties such as approximation properties and estimation error of both eDNNs and iDNNs, and (4) we implement various DNNs over a large class of manifolds under simulations and real datasets, including eDNNs, iDNNs and \emph{tangential deep neural networks (tDNNs)}, which is a special case of iDNNs with only one tangent space.

The rest of the paper is organized as follows. In Section 2, we introduce the eDNNs on manifolds and study their theoretical properties. In Section 3, we propose the iDNNs on manifolds that take into account the intrinsic geometry of the manifold. The simulation study and the real data analysis are carried out in Section 4. Our work ends with a discussion.

\section{Extrinsic deep  neural networks (eDNNs) on manifolds}
\label{sec-manifold}
\subsection{eDNNs and equivariant embeddings}

Let $M$ be a $d$-dimensional manifold. Let $(x_i, y_i)$, $i=1,\ldots, n$ be a sample of data from some regression model with input $x_i\in \mathcal X= M$ and output $y_i\in \mathcal Y=\R$, and we propose deep neural networks for learning the underlying function $f: M\rightarrow \mathbb{R}$. The output space can be $\mathcal Y=\{1,\ldots, k \}$  for a classification problem. In this work, we propose to develop two  general deep neural network architectures on manifolds based on an extrinsic and an intrinsic framework, respectively.  The first framework employs  an equivariant  embedding of  a manifold  into the Euclidean space and builds a deep neural network on its image after embedding, which is the focus of this section, while the intrinsic framework utilizes  Riemannian or intrinsic geometry of the manifold for designing the deep neural networks (Section \ref{sec-idnn}).  Our initial focus will be on  proposing appropriate analogs of feed-forward neural networks on manifolds which are popular DNNs in the Euclidean space and  suitable objects for theoretical analysis.  The theoretical properties of the proposed geometric DNNs will be  studied.

Before describing our proposed frameworks, we introduce our mathematical definition of DNNs and related classes. A DNN $\tilde f$ with depth $L$  and a width vector $\mathbf{p} = (p_0 ,\cdots, p_{L+1} ) \in \mathbb{N}^{L+2}$ is a function of the form
\begin{align}
\label{eq-ednnn}
\tilde f(\tilde x):=A_{L+1}\circ \sigma_L\circ A_L\circ\dots\circ\sigma_1\circ A_1(\tilde x),
\end{align}
where $A_l: \R^{p_{l-1}}\rightarrow \R^{p_l}$ is an affine linear map defined by 
$A_l(\tilde x)=\boldsymbol{W}_l\tilde x+\boldsymbol{b}_l$ for $p_l\times p_{l-1}$ weight matrix $\boldsymbol W_l$ and $p_l$ dimensional bias vector $\boldsymbol b_l$, and $\sigma_l: \R^{p_{l}}\rightarrow \R^{p_l}$ is an element-wise nonlinear activation map with the ReLU activation function $\sigma(z)=\max \{0, z\}$ as a popular choice. We referred to the maximum value $\max_{j=1,\dots, L }p_j$ of the width vector as the width of the DNN. We denote $\boldsymbol{\theta}$ as the collection of all weight matrices and bias vectors: 
$\boldsymbol \theta:=\left ((\boldsymbol{W}_1, \boldsymbol{b}_1),\ldots,  (\boldsymbol{W}_{L+1}, \boldsymbol{b}_{L+1})  \right),$ the parameters of the  DNN. Moreover, we denote by $\|\boldsymbol{\theta}\|_0$ the number of non-zero parameter values (i.e., the sparsity) and by $\|\boldsymbol{\theta}\|_\infty$ the maximum of parameters. We denote by  $\mathcal{F}(L, (p_0\sim P\sim p_{L+1}), S, B)$ the class of DNNs with depth $L$, input dimension $p_0$, width $P$, output dimension $p_{L+1}$, sparsity $S$ and the maximum of parameters $B$. For simplicity, if the input and output dimensions are clear in the context, we write $\mathcal{F}(L, P, S, B)=\mathcal{F}(L, (p_0\sim P\sim p_{L+1}), S, B)$.

Let $J: M\rightarrow\R^D$ be an embedding of $M$ into some higher dimensional Euclidean space $\mathbb{R}^D$ ($D\geq d$) and denote the image of the embedding as $\tilde{M} = J(M)$. By definition of an embedding, $J$ is a smooth map such that its differential  $dJ: T_xM\rightarrow T_{J(x)}\R^D$ at each point $x\in M$ is an injective map from its tangent space $T_xM$  to  $T_{J(x)}\R^D$, and $J$ is a homeomorphism between $M$ and its image $\tilde{M}$. Our idea of building \emph{an extrinsic DNN} on manifold relies on building a DNN on the image of the manifold after the embedding. The geometry of the manifold of $M$ can be well-preserved with a good choice of embedding, such as an equivariant embedding which will be defined rigorously in Remark \ref{remark:embedding} below. The extrinsic framework has been adopted for the estimation of Fr\'echet means \citep{ linclt}, regression on manifolds \citep{linjasa}, and construction of Gaussian processes on manifolds \citep{lin2019}, which have enjoyed some notable features such as ease of computations and accurate estimations.  

The key idea of proposing an extrinsic feedforward neural network on a manifold  $M$ is to build a one-to-one version of its image after the embedding.  More specially, we say that $f$ is an \emph{extrinsic deep  neural network (eDNN)} if $f$ is of the form
\begin{align}
\label{embedF}
f(x) = \tilde f(J(x)), 
\end{align}
with a DNN $\tilde{f}$. We denote the eDNN class induced by $\mathcal{F}(L, P, S, B)$ as 

$$\mathcal F_{eDNN}(L,P,S,B):=\{f=\tilde{f}\circ J:\tilde{f}\in\mathcal{F}(L, P, S, B)\}.$$

The extrinsic framework is very general and works for any manifold where a good embedding, such as an equivariant embedding, is available.  Under this framework, training algorithms in the Euclidean space, such as the stochastic gradient descent (SGD) with backpropagation algorithms, can be utilized working with the data $(J(x_i), y_i)$,  $i=1,\ldots ,n$,  with the only additional computation burden potentially induced from working higher-dimensional ambiance space.  In our simulation Section \ref{sec-sim}, the extrinsic DNN yields  better accuracy than the Naive Bayes classifier, kernel SVM, logistic regression classifier, and the random forester classifier for the planar shape datasets. Due to its simplicity and generality, there is a potential for applying eDNNs in medical imaging  and machine vision  for broader scientific impacts. 

\begin{remark}
In \cite{SchmidtHieber2019DeepRN} and \cite{2019chen}, a feedforward neural network was used for nonparametric regression on a lower-dimensional  submanifold embedded in some higher-dimensional ambient space. It showed that with appropriate conditions on the neural network structures, the convergence rates of the ERM  would depend on the dimension of the submanifold $d$ instead of the dimension of the ambient space $D$.  In their framework, they assume the geometry of the submanifold is unknown.  From  a conceptual point of view, our extrinsic framework can be viewed as a special case of theirs by ignoring the underlying geometry. In this case, the image of the manifold $\tilde M=J(M)$ can be viewed as a submanifold in $\R^D$, so their results follow. On the other hand, our embedding framework allows us to work with very complicated manifolds, such as the quotient manifolds for which no natural  ambient coordinates are available.   An example is the planar shape which is the quotient of a typically high-dimensional sphere  consisting of orbits of equivalent classes, with the submanifold structure only arising after the embedding. And such an embedding is typically not isometric. 

In  \cite{2019chen}, the charts were constructed by intersecting small balls in $\R^D$ with the submanifold $M$. In our case, we provide explicit charts of the submanifold based on the knowledge of the geometry of the original manifold $M$ and the embedding map $J$ that works with the ambient coordinates in $\R^D$.  
\end{remark}

\begin{remark}
\label{remark:embedding}
One of the essential steps in employing an eDNN is the choice of the embedding $J$, which  is generally not unique. It is desirable to have  an embedding that preserves as much geometry as possible.  An equivariant embedding  is one type of embedding that preserves a substantial amount of geometry.  Figure \ref{fig-embed} provides a visual illustration of equivariant embedding.  Suppose $M$ admits an action of a (usually `large') Lie group $H$. Then we say that $J$ is an equivariant embedding if we can find a Lie group homomorphism $\phi: H\rightarrow GL(D, \mathbb R)$ from $H$ to the general linear group $GL(D, \mathbb R)$ of degree $D$ acting on $\tilde M$ such that
\begin{align*}
J(hp)=\phi(h)J(p)
\end{align*}
for any $h\in H$ and $p\in M$.  The definition seems technical at first sight. However, the intuition is clear. If a large group $H$ acts on manifolds such as by rotation before embedding, such an action can be preserved via $\phi$  on the image $\tilde M$, thus potentially preserving many of the geometric features of $M$, such as its symmetries. Therefore, the embedding is geometry-preserving in this sense. For the case of the planar shape, which is a collection of shapes consisting of $k$-landmarks  modular Euclidean motions such as rotation, scaling, and translation, which is a quotient manifold of a sphere of dimension $S^{2k-3}$, and the embedding can be given by the Veronese-whitnning embedding which is equivariant under the special unitary group. Another example that's less abstract to understand is the manifold of symmetric positive definite matrices whose embedding can be given as the  $\log$ map (the matrix $\log$ function) into the space of symmetric matrices, and this embedding is equivariant with respect to the group action of the general linear group via the conjugation group action. 
See Section \ref{sec-sim} for some concrete examples of equivariant embeddings for well-known manifolds, such as the space of the sphere, symmetric positive definite matrices, and planar shapes. 
\end{remark}

\begin{figure}
\includegraphics[width=1.2\linewidth, angle=0]{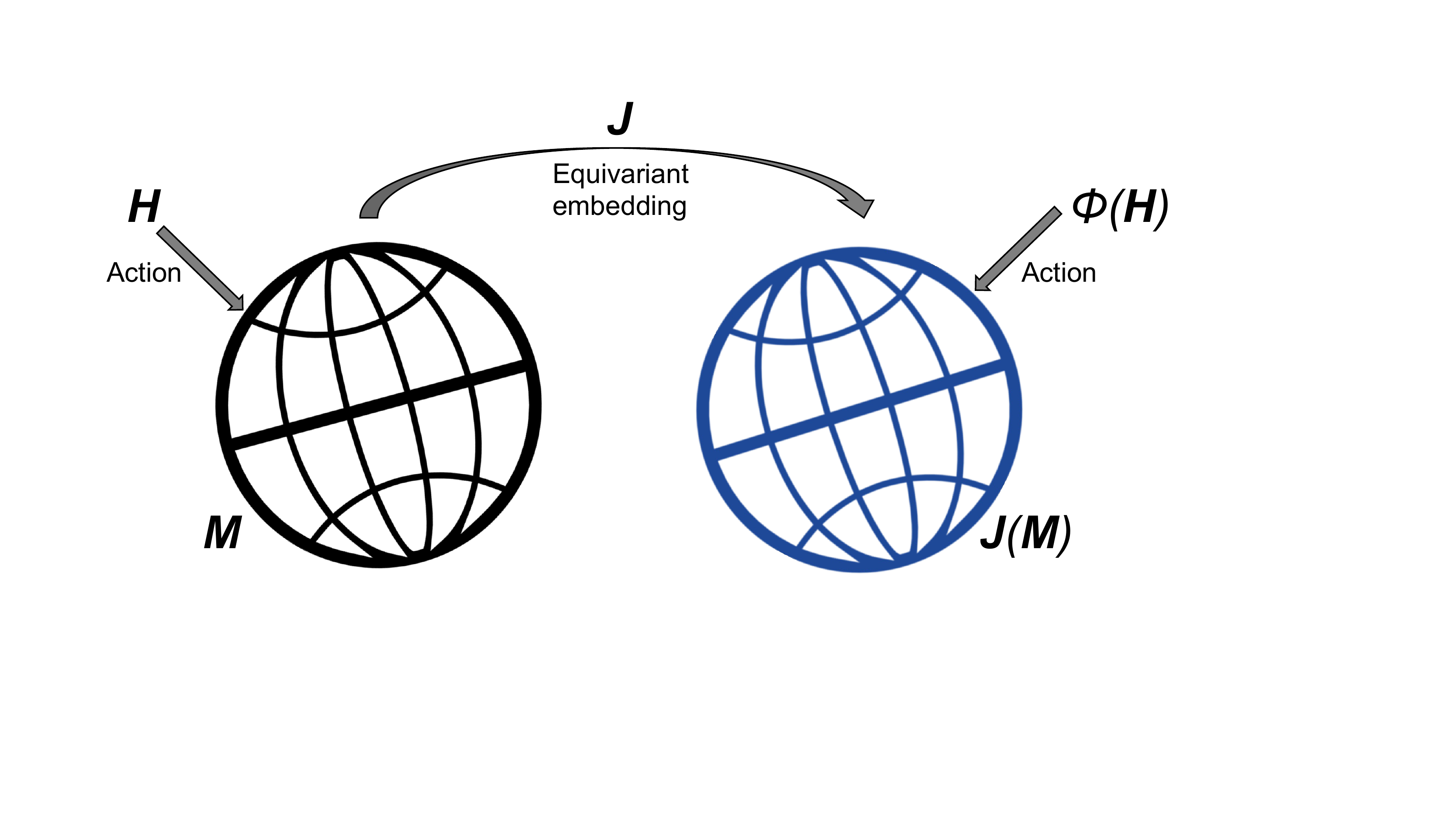}
\caption{\footnotesize An simple illustration of equivariant embeddings}
\label{fig-embed}
\end{figure}

\subsection{Approximation analysis for eDNNs}

In this section, we study the ability of the eDNN class in approximating an appropriate smooth class of functions on manifolds. First, we define the ball  of $\beta$-H\"older functions on a set $U\in\mbR^D$ with radius $K$ as
    \begin{align*}
    \mathcal{C}_D^\beta(U,K) = \{f: \| f\|_{\mathcal{C}^\beta_D(U)}\le K\},
    \end{align*}
where $\|\cdot\|_{\mathcal{C}^\beta_D(U)}$ denotes the $\beta$-H\"older norm defined as
    \begin{align*}
   \|f\|_{\mathcal{C}^{\beta}_D(U)}
   =\sum_{m\in\mbN_0^D:\|m\|_1\le \lfloor\beta\rfloor}\|\partial^{m}f\|_{\infty}
        +\sum_{m\in\mbN_0^D:\|m\|_1=  \lfloor\beta\rfloor}\sup_{x_1,x_2\in U, x_1\neq x_2 }\frac{|\partial^{m}f(x_1)-\partial^{m}f(x_2)|}{\|x_1-x_2\|_\infty^{\beta-  \lfloor\beta\rfloor}}.
    \end{align*}
Here,  $\partial^{m}f$ denotes the partial derivative of $f$ of order $m$ and $\mbN_0:=\mbN\cup\{0\}$. To facilitate smooth function approximation on manifolds, following \cite{SchmidtHieber2019DeepRN}, we  impose an additional smooth assumption on local coordinates which project inputs in an ambient space to a lower dimensional space.

\begin{definition}
We say that a compact $d$-dimensional manifold $M\subset \mbR^D$ has smooth local coordinates if there exist charts $(V_1, \psi_1),\dots,(V_r, \psi_r)$, such that for any $\gamma>0$, $\psi_j \in \mathcal{C}^\gamma_D(\psi_j(V_j))$.
\end{definition}

The next theorem reveals the approximation ability of the eDNN class. For a measure of approximation, we consider the sup norm defined as $ \| f_1-f_2\|_{L^\infty(M)}:=\sup_{x\in M}|f_1(x)-f_2(x)|$ for two functions $f_1,f_2:M\to \mbR$.

\begin{theorem}
\label{thm:ednn:approx}
Let $M\subset \mathbb{R}^D$  be a $d$-dimensional compact manifold and $J:M\to \mbR^D$ be an embedding map. Assume that $J(M)$ has smooth local coordinates.  %and $B$ be the upper bound of the embedding $ \|J\| \leq B$. 
Then there exist positive constants $c_1,c_2$ and $c_3$  depending only on  $D,d,\beta,K$ and the surface area of $M$ such that for any $\eta \in (1, 0)$,
    \begin{align*}
        \sup_{f_0:M\to [-1,1] \textup{ s.t }  f_0\circ J^{-1}\in C^\beta_D(J(M), K)} \inf_{f \in \mathcal{F}_{eDNN}(L, P, S, B=1)}\|f - f_0\|_{L^{\infty}(M)} \leq \eta
    \end{align*}
with $L \leq c_1\log\frac{1}{ \eta}$,  $P \leq c_2 \eta^{-\frac{d}{\beta}}$ and $S\le c_3 \eta^{-\frac{d}{\beta}}\log\frac{1}{ \eta}$.
\end{theorem}

\begin{proof}
Let $\tilde{f}_0 = f_0\circ J^{-1}$, then $\tilde{f}_0$ is a function on the $d$-dimensional manifold $\tilde{M}=J(M) \subset\mbR^D$. Since $\tilde{M}$ has smooth local coordinates, we can apply Theorem 2 in \cite{SchmidtHieber2019DeepRN}, there exists a network $\tilde{f}\in\mathcal{F}(L, (D\sim P \sim 1), S,1)$ such that $\|\tilde{f}-\tilde{f}_0\|_{L^{\infty}(\tilde{M})}<\eta$ with $L \leq c_1\log\frac{1}{ \eta}$,  $P \leq c_2 \eta^{-\frac{d}{\beta}}$ and $S\le c_3 \eta^{-\frac{d}{\beta}}\log\frac{1}{ \eta}$ for some $c_1>0,c_2>0$ and $c_3>0$. Now, let $f = \tilde{f}\circ J\in \mathcal{F}_{eDNN}(L, (D\sim P \sim 1), S,1)$. Then
\begin{equation*}
\|f-f_0\|_{L^{\infty}(M)} = \|\tilde{f}\circ J-\tilde{f}_0\circ J \|_{L^{\infty}(M)}  = \|\tilde{f}-\tilde{f}_0\|_{L^{\infty}(\tilde{M})}.
\end{equation*}
Therefore, we get the desired result.
\end{proof}

\subsection{Statistical risk analysis for eDNNs}

In this section, we  study the statistical risk of the empirical risk minimizer (ERM) based on the eDNN class. We assume the following regression model
    \begin{align}
    \label{eq:regmodel}
        y_i = f_0(x_i) + \epsilon_i
    \end{align}
for $i=1,\dots, n$, where $x_1,\dots, x_n\in M$ are i.i.d inputs following a distribution $P_x$ on the manifold and $\epsilon_1,\dots, \epsilon_n$ are i.i.d. sub-Gaussian errors. We consider the ERM over the eDNN class such that
\begin{equation}
    \label{eq-erm}
\hat{f}_{eDNN}=\underset{f\in \mathcal F_{eDNN} (L, P, S, B) }{\argmin}\frac{1}{n}\sum_{i=1}^n(y_i-f(x_i))^2.
\end{equation}
A natural question to ask is whether the  ERM type of estimators such as $\hat{f}_{n}$ defined above achieve minimax optimal estimation of $\beta$-H\"older smooth functions on manifolds, in terms of the excess risk  
    \begin{align*}
        R(\hat{f}_{eDNN}, f_0)
        :=E(\hat{f}_{eDNN}(x)-f_0(x))^2
    \end{align*}
where the expectation is taken over the random variable $x\sim P_x$.

\begin{theorem}
\label{thm:ednn:risk}
Assume the model (\ref{eq:regmodel}) with a $d$-dimensional compact manifold $M\subset \mathbb{R}^D$ and an embedding map $J:M\to\mbR^D$. Moreover, assume that $J(M)$ has smooth local coordinates. Then the ERM estimator $\hat{f}_{eDNN}$ over the eDNN class $\mathcal{F}_{eDNN}(L, P, S, B=1)$ in (\ref{eq-erm}) 
with $L \asymp \log(n)$,  $n\gtrsim P\gtrsim n^{d/(2\beta+d)}$ and $S\asymp n^{d/(2\beta+d)}\log n$ satisfies
    \begin{align*}
       \sup_{f_0:M\to [-1,1] \textup{ s.t }  f_0\circ J^{-1}\in C^\beta_D(\mbR^D, K)}  R(\hat{f}_{eDNN},f_0)\lesssim n^{-\frac{2\beta}{2\beta+d}}\log^3n.
    \end{align*}
\end{theorem}

\begin{proof}
For any $\tilde{f_1},\tilde{f_2}\in\mathcal{F}(L, P, S, B=1)$, we have  $\|\tilde{f_2}\circ J-\tilde{f_2}\circ J\|_{L^\infty(M)}= \|\tilde{f_2}-\tilde{f_2}\|_{L^\infty(\tilde{M})}\le \|\tilde{f_2}-\tilde{f_2}\|_{L^\infty(\mbR^D)}$. Hence the entropy of the eDNN class $\mathcal{F}_{eDNN}(L, P, S, B=1)$ is bounded by that of $\mathcal{F}(L, P, S, B=1)$. Thus, by Lemmas 4 and 5 of \cite{schmidt-hieber2020}, we have
\begin{equation*}
R\left(\hat{f}_{eDNN}, f_{0}\right) \lesssim \inf _{f \in \mathcal{F}_{eDNN}(L, P, S, B=1)} \|f-f_0\|^2_{L^{\infty}(M)}+\frac{(S+1) \log \left(2 n(L+1)P^{2L}(D+1)^2\right)+1}{n}.
\end{equation*}
Therefore, by Theorem \ref{thm:ednn:approx}, if we take $L, P$ and $S$ as in the theorem, we get the desired result.
\end{proof}

\section{Intrinsic deep neural networks (iDNNs) on manifolds} 
\label{sec-idnn}

\subsection{The iDNN architectures on a Riemannian manifold}
Despite the generality and  computational advantage enjoyed by eDNNs on manifolds proposed in the previous section, one potential drawback is that an embedding is not always available on complex manifolds such as some intrinsic structure spatial domains. In this section, we propose a class of intrinsic DNNs on manifolds (iDNNs) by employing the intrinsic geometry of a manifold to utilize its exponential and log maps with respect to a Riemannian structure. Some works construct a DNN on the manifold via mapping the points on the manifold to a \emph{single tangent space} (e.g., with respect to some central points of the data)  or proposing DNNs on specific manifolds, in particular, matrix manifolds \cite{deepspd, matrixback-prop}. Using a DNN on a single tangent space approximation cannot provide a good approximation of a function on the whole manifold.  Below we provide a rigorous framework for providing a local approximation of a function on a Riemannian manifold via Riemannian exponential and logarithm maps and thoroughly investigate their theoretical properties. 

The key ideas here are to first cover the manifold with images of the subset of tangent spaces $U_1,\ldots, U_k$ under the exponential map, approximate a local function over the tangent space using DNNs, which are then patched together via the transition map  and a partition  of unity on the Riemannian manifold. Specifically, let $\{x_1, \dots, x_K\in M\}$ be a finite set of points, such that for an open set of subsets $U_k\subset T_{x_k}M$ with $k=1, \dots, K,$ one has $\bigcup_{k=1}^K\exp_{x_k}(U_k)=M.$ Namely, one has $\big\{\big(\exp_{x_k}(U_k), \  \exp_{x_k}\big), \  \ k=1 \dots, K\big\}$ as the charts of the manifold $M$.

For each $k=1, \dots, K$ one has orthonormal basis $v_{k1}, \dots, v_{kd}\in T_{x_k}M$ and respectively the normal coordinates of $x\in\exp_{x_k}(U_k)$
\begin{align*}
v_k^j(x)=\big\langle\log_{x_k}x, v_{kj}\big\rangle \quad \textrm{for} \quad j=1, \dots, d.
\end{align*}
Thus
\begin{align*}
v_k(x)=\big(v_k^1(x), \dots, v_k^d(x)\big)=\sum_{j=1}^dv_k^j(x)v_{kj}\in T_{x_k}M.
\end{align*}
The normal coordinate allows one to perform elementwise non-linear activation to tangent vectors easily. 
 For example, any $1\leq k<l \leq K$ one has the transition map on $\exp_{x_l}(U_l)\cap\exp_{x_k}(U_k)$
  \begin{align*}
v_k^j(x)=\big\langle\log_{x_k}x, v_{kj}\big\rangle = \big\langle\log_{x_k}\exp_{x_l}v_l(x), v_{kj}\big\rangle \quad \textrm{for} \quad j=1, \dots, d.
  \end{align*}

A compact manifold $M$ always  admits  a \emph{finite partition of unity} $\big\{\tau_k, \ k=1, \dots, K\big\}$, $\tau_k(\cdot): M\rightarrow \mathbb R_+$ such that $\sum _{k=1}^K\tau_k (x)=1,$ and for every $x\in M$ there is a neighbourhood of $x$ where all but a finite number of functions are $0$ (e.g., Proposition 13.9 of \cite{loring2011introduction}). Therefore, for each function $f:M\rightarrow\mathbb{R}$, we can write
\begin{align}
\label{eq-idnnarchitecture}
f(x)=\sum_{k=1}^K\tau_k(x)f\Big(\exp_{x_k}\big(\log_{x_k}x\big)\Big)\doteq\sum_{k=1}^K\tau_k(x) f_k(\log _{x_k}(x)).
\end{align}
As a result,  one can model the compositions $f_k=f\circ\exp_{x_k}:U_k\rightarrow\mathbb{R}$ instead of $f,$ for which we propose to use DNN. This idea gives rise to our iDNN architecture  $f(x)=\sum_{k=1}^K\tau_k(x)f_k\left(\log _{x_k}(x)\right)$. Figure \ref{fig:idnn} illustrates the core ideas of the iDNN architecture. Given a set of points $\{x_1,\dots,x_K\}\subset M$, we define the iDNN class with depth $L$, width $P$, sparsity $S$ and the maximum of parameters $B$ as
\begin{align}
        \mathcal{F}_{iDNN}(L,P, S, B)=\left\{\sum_{k=1}^K\tau_k(x)f_k\left(\log _{x_k}(x)\right):f_k\in \mathcal{F}(L,(d\sim P\sim 1), S, B)\right\}.
\end{align}

\begin{figure}[ht]
    \centering
    \includegraphics[width=15cm]{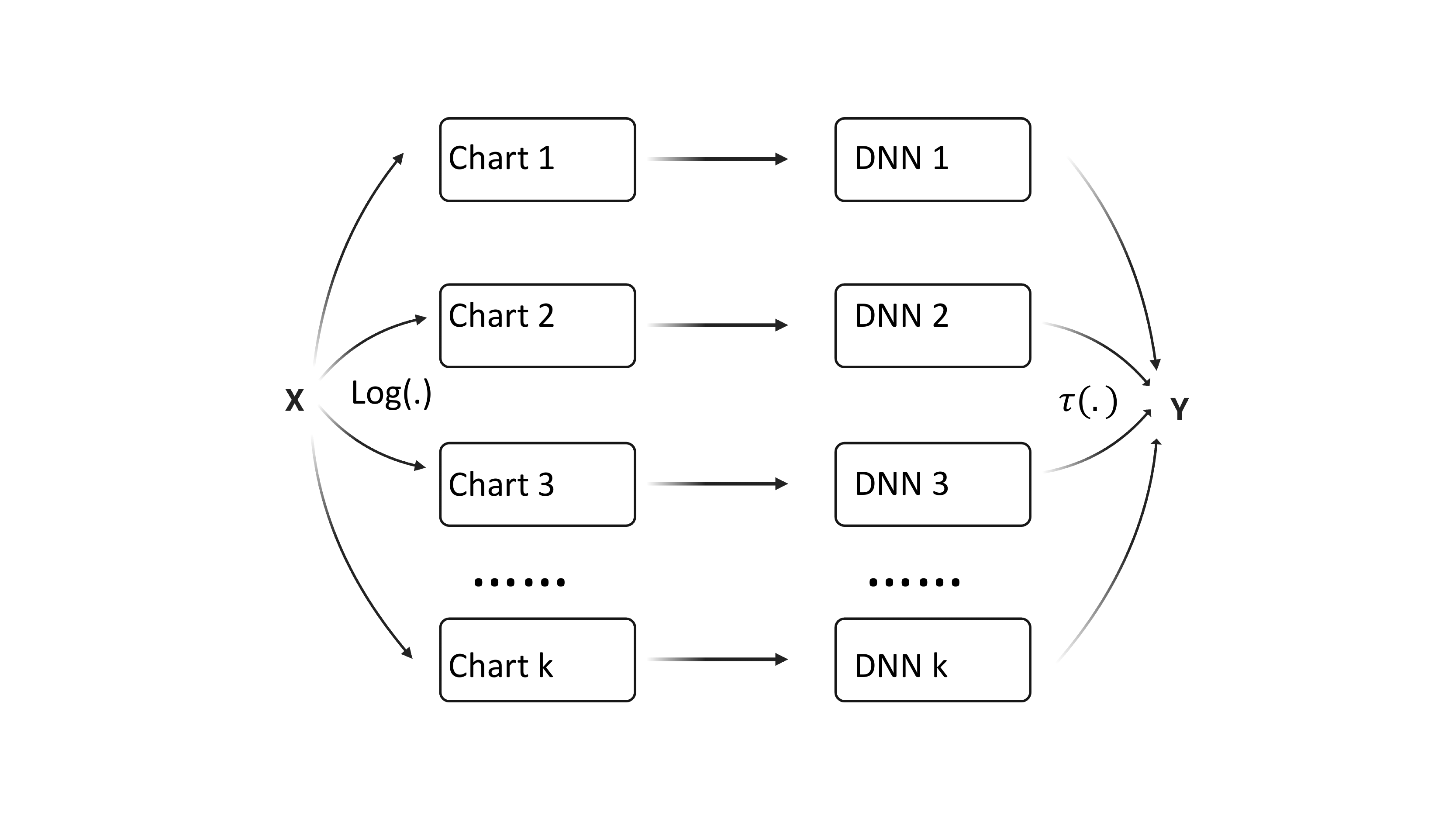}
    \caption{\footnotesize The iDNN architecture on a Riemannian manifold $M$. Given the base points $\{x_1, \dots, x_K\in M\}$ and the charts $\big\{U_k\subset T_{x_k}M, k=1 \dots, K\big\}$ on the manifold $M$, the input data $X$ is mapped to the $kth$ chart $U_k$ after the log map $\log_{x_k}(.)$. Afterward, the transformed data is fed into the DNN $f_k$ on each chart $k$. The final prediction $Y$ is given by the partition of unity $\tau(.)$ as $Y = \sum_{k=1}^K\tau_k(x)f_k\left(\log _{x_k}(x)\right)$. }
    \label{fig:idnn}
\end{figure}

\subsection{Approximation analysis for iDNNs}
In this section, we investigate the approximation theory for the iDNN for smooth functions on manifolds. 

\begin{theorem}
\label{thm:idnn:approx}
Let $M\subset \mathbb{R}^D$ be a $d$-dimensional compact manifold. Assume that $\exp_{x_k}\in \mathcal{C}_D^{\gamma}(U_k)$ for $\gamma>\beta$ for every $k=1,\dots, K$.  Then there exist positive constants $c_1,c_2$ and $c_3$  depending only on $D,d,\beta,K$ and the surface area of $M$  such that for any $\eta \in (1, 0)$,
\begin{align*}
 \sup_{f_0:M\mapsto[-1, 1] \text{ s.t. } f_0\in\mathcal{C}_D^{\beta}(M, K)}\inf_{f\in \mathcal{F}_{iDNN}(L,P, S, B=1)}\|f-f_0\|_{L^{\infty}(M)}\leq \eta.    
\end{align*}
with $L \leq c_1\log\frac{1}{ \eta}$,  $P \leq c_2 \eta^{-\frac{d}{\beta}}$ and $S\le c_3 \eta^{-\frac{d}{\beta}}\log\frac{1}{ \eta}$.
\end{theorem}

\begin{proof}
We construct a DNN approximating $f_{0k}=f_0\circ \exp_{x_k}$ for each $k=1,\dots, K$. Note that $f_{0k}$ is $\beta$-H\"older smooth by assumption. Therefore, by Theorem 1 in \cite{SchmidtHieber2019DeepRN}, there exist DNNs $f_1,\dots, f_K\in\mathcal{F}(L, (d\sim P \sim 1), S,1)$ such that $\|f_k-f_{0k}\|_{L^{\infty}(U_k)}<\eta$ with $L \leq c_1\log\frac{1}{ \eta}$,  $P \leq c_2 \eta^{-\frac{d}{\beta}}$ and $S\le c_3 \eta^{-\frac{d}{\beta}}\log\frac{1}{ \eta}$ for some $c_1>0,c_2>0$ and $c_3>0$. Now, let $f = \sum_{k=1}^K\tau_k(x)f_k(\log _{x_k}(x))\in \mathcal{F}_{iDNN}(L,P, S,1)$. Then
    \begin{align*}
        \|f-f_0\|_{L^{\infty}(M)}
        &=\sup_{x\in M }\left|\sum_{k=1}^K\tau_k(x)f_k(\log _{x_k}(x))-\sum_{k=1}^K\tau_k(x)f_{0k}(\log _{x_k}(x))\right|\\
       &\le \sup_{x\in M }\sum_{k=1}^K\tau_k(x)\left|f_k(\log _{x_k}(x))-f_{0k}(\log _{x_k}(x))\right|\\
        &\le \max_{1\le k\le K}\left\|f_k-f_{0k}\right\|_{L^{\infty}(U_k)}<\eta
    \end{align*}
which completes the proof.
\end{proof}

\begin{remark}
    \cite{SchmidtHieber2019DeepRN} and \cite{2019chen} propose feedforward neural networks on a manifold that's embedded in a higher-dimensional Euclidean space.  In the approximation theory  of \cite{SchmidtHieber2019DeepRN} and \cite{2019chen}, they utilize local charts and partition of unities, but due to the unknown geometry of the manifold, they need to use DNNs to approximate the local charts $\psi_j$s, the partition of unities functions as well as the mappings $f\circ \psi_j^{-1}$.  Under our iDNN framework, we utilize the Riemmanian geometry of the manifold and the $\log$ map.  Further, the partition of utilities functions can be constructed so there is no need to approximate them with DNNs.
    
\end{remark}

\subsection{Statistical risk analysis for iDNNs}

In this section, we  study the statistical risk of the ERM over the iDNN class given by
    \begin{equation}
        \label{eq-erm-idnn}
        \hat{f}_{iDNN}=\underset{f\in \mathcal F_{iDNN} (L, P, S, B) }{\argmin}\frac{1}{n}\sum_{i=1}^n(y_i-f(x_i))^2.
    \end{equation}
for the nonparametric regression model (\ref{eq:regmodel}) where the true function $f_0$ is $\beta$-H\"older smooth on a manifold. The following theorem shows that the iDNN estimator attains the optimal rate. We omit the proof since it is almost the same as the proof of Theorem \ref{thm:ednn:risk} except using the approximation result for the iDNN class given in Theorem \ref{thm:idnn:approx}.

\begin{theorem}
Assume the model (\ref{eq:regmodel}) with a $d$-dimensional compact manifold $M$ isometrically embedded in $\mbR^D$. Then the ERM estimator $\hat{f}_{iDNN}$ over the iDNN class $\mathcal{F}_{iDNN}(L, P, S, B=1)$ in (\ref{eq-erm-idnn}) 
with $L \asymp \log(n)$,  $n\gtrsim P\gtrsim n^{d/(2\beta+d)}$ and $S\asymp n^{d/(2\beta+d)}\log n$ satisfies
    \begin{align*}
       \sup_{f_0:M\mapsto[-1, 1] \text{ s.t. } f_0\in\mathcal{C}_D^{\beta}(M, K)}R(\hat{f}_{iDNN},f_0)\lesssim n^{-\frac{2\beta}{2\beta+d}}\log^3n.
    \end{align*}
\end{theorem}

\begin{proof}
For any two iDNNs $f(\cdot) = \sum_{k=1}^K\tau_k(\cdot)f_k(\log _{x_k}(\cdot))$ and $f'(\cdot)  = \sum_{k=1}^K\tau_k(\cdot)f_k'(\log _{x_k}(\cdot))$ in $ \mathcal{F}_{iDNN}(L,P, S,B)$, we have
    \begin{align*}
        \|f-f'\|_{L^\infty(M)}
       &\le \sup_{x\in M }\sum_{k=1}^K\tau_k(x)\left|f_k(\log _{x_k}(x))-f_{0k}(\log _{x_k}(x))\right|\\
        &\le \max_{1\le k\le K}\left\|f_k-f_{0k}\right\|_{L^\infty(\mbR^d)}.
    \end{align*}
Therefore, the entropy of $ \mathcal{F}_{iDNN}(L,P, S,B)$ is bounded by the $K$-times of the entropy of the class $ \mathcal{F}(L,P, S,B)$. So by the same way as in the proof of Theorem \ref{thm:ednn:risk}, we get the desired.
\end{proof}

\section{Simulations study and real data analysis}
\label{sec-sim}
 Applications will illustrate the practical impact and utilities of our methods to simulated data sets and some important real data sets, such as in the  context of the AFEW database, HDM95 database, the ADHD-200 dataset, an HIV study, and others.  The proposed eDNNs, tDNNs, and iDNNs will be applied to learning problems such as regression and classification on various manifolds, including the sphere, the planar shapes, and the manifold of symmetric positive definite matrices, which are the most popular classes of manifolds encountered in  medical diagnostics using medical imaging and  image classification in digital imaging analysis. For the eDNN models, we list explicit embeddings below and the corresponding lie groups that act on them equivariantly. For the iDNN models, we elaborate the exponential map and inverse-exponential (log) map on those manifolds. As mentioned before, the tDNN model is the special case of the iDNN model when $K=1$, which utilizes the exponential map and inverse-exponential map as well.

 \subsection{Sphere}
One of the simplest manifolds of interest is the sphere in particular in directional statistics and spatial statistics \citep{fisherra53, maridajpuu, fisher87, jun2008, chunfeng}. Statistical analysis of  data from  the two-dimensional sphere $S^2$, often called directional statistics, has a fairly long history \citep{fisherra53, maridajpuu, fisher87}. Modeling on the sphere has also received recent attention  due to  applications in spatial statistics, for example, global models for climate or satellite data \citep{jun2008, chunfeng}.  

To build the eDNN on the sphere, first note that $S^d$ is a submanifold of $\R^{d+1}$, so that the inclusion map $J$ serves as a natural embedding of $S^d$ into $\mathbb R^{d+1}$.  It is easy to check that $J$ is an equivariant  embedding with  respect to the Lie group $H=SO(d+1)$, the group of $d+1$ by $d+1$ special orthogonal matrices. Intuitively speaking, this embedding preserves a lot of  symmetries of the sphere. On the other hand, one can use the geodesics (in this case, the big circles on the sphere) for which the closed-form exponential map and inverse-exponential map are available to construct the iDNN model. Furthermore, given the base points $x_i, i=1,...,k$, one has $\tau(x) = e^{-\frac{1}{1 - \|x - x_i\|^2}}$ by utilizing the bump function on the sphere. 

In this simulation study, we consider the classification problem in terms of the Von Mises-Fisher distribution (MF) on the sphere $S^2$, which has the following density: 

\begin{equation}
f_{\mathrm{MF}}(y ; \mu, \kappa) \propto \exp \left(\kappa \mu^{T} y\right),
\end{equation}
where $\kappa$ is a concentration parameter with $\mu$ a location parameter. Then we simulate the data from $K$ different classes on the sphere $ S^d$ via a mixture of MF as: 
\begin{align}
 & u_{i1},...,u_{i10} \sim \mathrm{MF}(\mu_i,\kappa_1), i = 1,..,K.\\
& m_{ij} \sim unif\{u_{k1},...,u_{k10}\}, \quad x_{ij} \sim \mathrm{MF}(m_{ij},\kappa_2), \\
& i = 1,..,K, \quad j= 1,...,N.
\end{align}

Here $x_{ij}$ is the $j$th sample from $i$th class, $\mu_i$ is the mean for the $i$th class, and $\kappa$ is the dispersion for all classes. We first generated $10$ means $u_{i1},...,u_{i10}$ from the $\mathrm{MF}$ distribution for $i$th class. Then for each class, we generated $N$ observations as follows: for each observation $x_{ij}$, we randomly picked $m_{ij}$ from $u_{k1},...,u_{k10}$ with probability $1/10$, and then generated a $\mathrm{MF}(m_{ij},\kappa_2)$, thus leading to a mixture of $\mathrm{MF}$ distribution. Moreover, $\kappa_1$ controls the dispersion of the intermediate variable $m_{ij}$ while $\kappa_2$ controls the dispersion of observations $x_{ij}$. Figure \ref{fig:sphere} shows observations from the mixture model on the sphere under different dispersions. 

\begin{figure*}[htbp]
    \centering
    \subfloat[$\kappa_1=10, \kappa_2 = 50$]{
    \includegraphics[width=.35\textwidth]{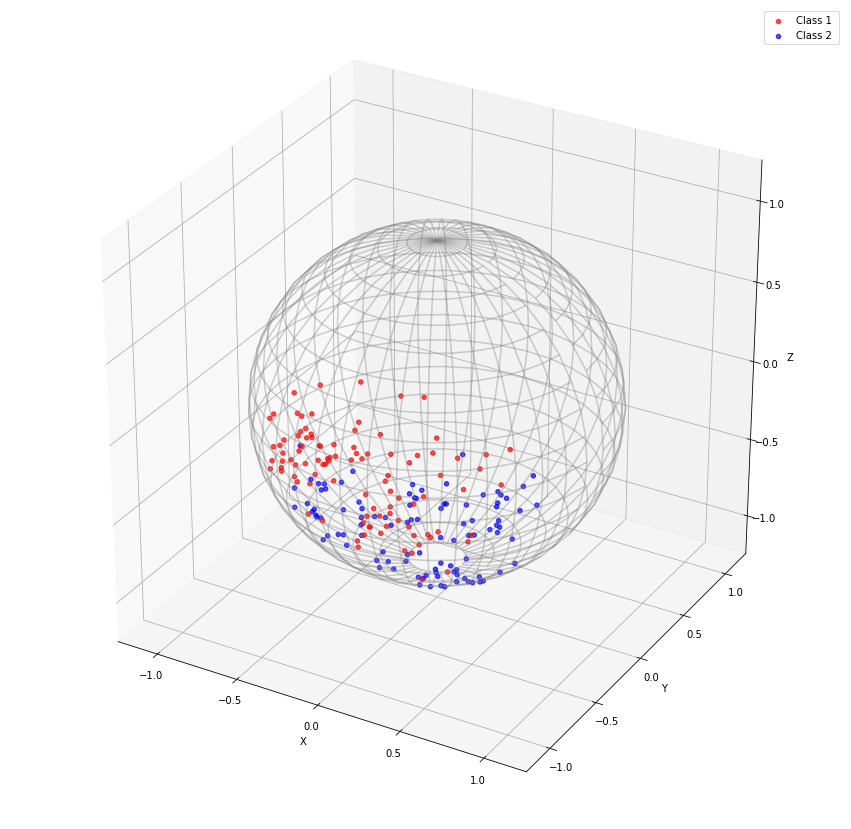}}
    \subfloat[$\kappa_1=8, \kappa_2 = 40$]{
    \includegraphics[width=.35\textwidth]{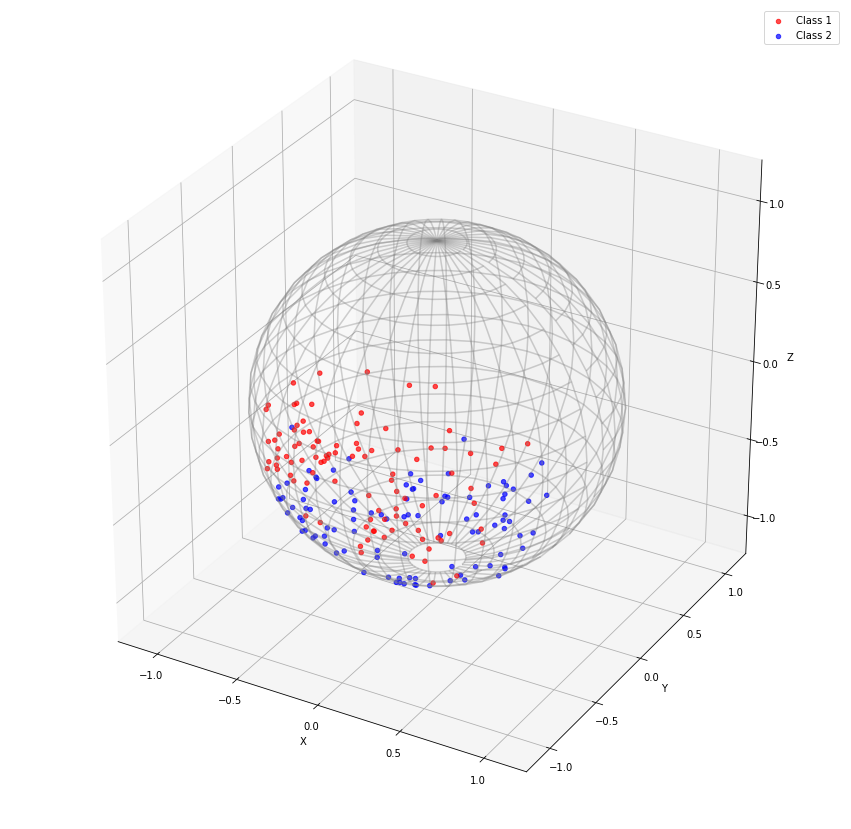}}
    \subfloat[$\kappa_1=4, \kappa_2 = 20$]{
    \includegraphics[width=.35\textwidth]{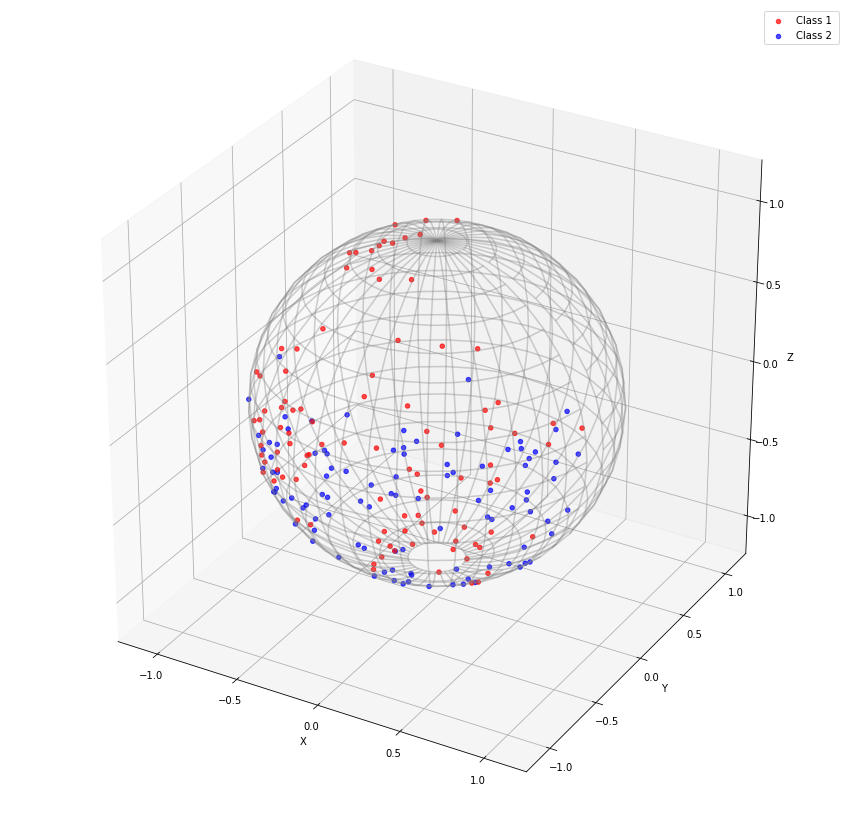}}\\
 
    \caption{Observations for $K =2$ classes from the mixture $\mathrm{MF}$ distribution, $N =100$. The nonlinear boundary between the two classes becomes hard to see with bare eyes due to the surging variance of the data as the $\kappa_1, \kappa_2$ dropping, which makes the classification problem harder.  }
  \label{fig:sphere}
\end{figure*}

In the following simulation, we follow the mixture model on the hyper-sphere $S^2, S^{10}, S^{50}$ with $K =2$, $N = 2000$, $\kappa_1 = 4$, $\kappa_2 = 20$ and divide the data into 75 percent training set and 25 percent test set. We repeat this split $50$ times. Then we compare the eDNN, tDNN, iDNN models to other competing estimators via the classification accuracy on the test set in Table \ref{tab:test-sphere}. 

For competing estimators, we consider the k-nearest neighbors (kNN), the random forest (RF), the logistic regression (LR), and the support vector machine (SVM) with the radial basis function (RBF) kernel. The tuning parameters in each method are selected by evaluation on a validation data set whose size is $25\%$ of the training set.

For all DNN models, we apply a network architecture of $5$ hidden layers with the numbers of widths $(100,100,100,100,100)$. The DNN model is the same as the eDNN model on Euclidean since the embedding map from the sphere to the higher Euclidean space is the identity map. In the tDNN model, we consider the $Frechet$ mean of the training set as the base point and transform all data in the batch to tangent vectors before feeding to the neural network. In the iDNN model, we consider the north and south poles $(\pm 1,0,..,0)$ as base points and use the neural network with the same structure for all tangent spaces. All models are trained with Adam optimizer \cite{kingma2014adam}. As shown in Table 1, our tDNN model and iDNN model outperform other competing estimators. Specifically, our tDNN models achieve the best accuracy $94.88 \pm 0.53$ and $97.13 \pm 0.39$ in the low dimensional cases. Our iDNN models obtained the best result $ 80.72 \pm 0.94$ and $68.43 \pm 1.20$ in the high dimensional spaces.

\begin{table}[htbp]
    \centering
    \caption{The test accuracy is calculated over $50$ random split. The $5$-layers network (with $100$ hidden nodes in each layer) is used for our DNN models in all experiments. Our tDNN model achieved the best result when the dimension was low $S^2, S^{10}$, while our iDNN is the best in high-dimension cases ($S^{50}, S^{100}$). Moreover, our tDNN, iDNN models show better accuracy than the classical DNN, especially in high-dimensional cases.  }
    \begin{tabular}{lccccc}
         & $S^2$ & $S^{10}$ & $S^{50}$ & $S^{100}$\\
        \hline
        $\text{DNN}$  & $ 94.12  \pm 0.67  $ & $ 96.22 \pm 0.63 $ & $ 75.93 \pm 1.07 $ & $62.53 \pm 1.35$\\
        $\text{tDNN}$ & $ \mathbf{94.88 \pm 0.53} $ & $\mathbf{ 97.13 \pm 0.39} $ & $ 80.07 \pm 0.95$ & $68.26 \pm 1.16$\\
        $\text{iDNN}$ & $94.69  \pm 0.65 $ & $  97.11\pm 0.41 $ & $ \mathbf{80.72 \pm 0.94} $ &$\mathbf{68.43 \pm 1.20} $ \\

        \hline
        $\text{kNN}$ & $ 92.16\pm 0.77  $ & $ 94.98 \pm 0.60$ & $ 69.18\pm 1.44$ & $56.24 \pm 1.30$\\
        $\text{LR}$& $ 92.98 \pm 0.76  $ & $ 88.64 \pm 0.76$ & $ 72.38 \pm 1.14$ & $66.73 \pm 1.37$\\
        $\text{RF}$& $ 93.66 \pm 0.83  $ & $89.93 \pm 0.65$ & $ 70.29 \pm 1.48$ & $62.29 \pm 1.45$\\
        $\text{SVM}$& $ 94.07 \pm 0.1  $ & $96.85 \pm 0.44$ & $ 79.38 \pm 1.15$ & $68.25 \pm 1.18$\\
    \end{tabular}
    \label{tab:test-sphere}
\end{table}

\subsection{The planar shape}
Let $z=(z_1,\ldots, z_k)$, with $z_1,\ldots, z_k\in \R^2$, be a set of $k$ landmarks. The planar shape $\Sigma_2^k$ is the collection of $z$'s modulo under the Euclidean motions, including translation, scaling, and rotation. One has $\Sigma_2^k=S^{2k-3}/SO(2)$,  the quotient of sphere by the action of $SO(2)$ (or the rotation),  the group of $2\times 2$ special orthogonal matrices;
%$\Sigma_2^k$ can be shown to be equivalent to the complex projective space $\mathbb C\mathbb P^{k-2}$.
A point in $\Sigma_2^k$ can be identified as the orbit of some $u\in S^{2k-3}$, which we denote as $\sigma(z)$. Viewing $z$ as a vector of complex numbers, one can embed $\Sigma_2^k$ into $S(k,\mathbb C)$, the space of $k\times k$ complex Hermitian matrices, via the Veronese-Whitney embedding (see, e.g., \cite{rabimono}):
\begin{equation}
\label{eq-planaremb}
J(\sigma(z))=uu^*=((u_i\bar{u}_j))_{1\leq, i,j\leq k}.
\end{equation}
One can verify that $J$ is equivariant (see \cite{kendall84}) with respect to the Lie group
$$H=SU(k)=\{A\in GL(k, \mathbb C), AA^*=I, \det(A)=I\},$$
with its action on $\Sigma_2^k$ induced by left multiplication.

We consider a planar  shape data set,  which involves measurements of a group of typically developing children and a group of children suffering the ADHD (Attention deficit hyperactivity disorder).  ADHD  is one of the most common psychiatric  disorders for children that can continue through adolescence and adulthood. Symptoms include difficulty staying focused and paying attention, difficulty controlling behavior, and hyperactivity (over-activity). In general, ADHD has three subtypes: (1) ADHD hyperactive-impulsive, (2) ADHD-inattentive, (3) Combined hyperactive-impulsive and inattentive  (ADHD-combined).  ADHD-200 Dataset (\url{http://fcon_1000.projects.nitrc.org/indi/adhd200/}) is a data set that  records both anatomical and resting-state functional MRI data of 776 labeled subjects across 8 independent imaging sites, 491 of which were obtained from typically developing individuals and 285 in children and adolescents with ADHD (ages: 7-21 years old).  
 The planar Corpus Callosum shape data are extracted, with 50 landmarks on the contour of the Corpus Callosum of each subject (see \cite{hongtu15}). See  Figure \ref{fig:landmark} for a plot of the raw landmarks of a normal developing child and an ADHD child)
 After quality control,  647 CC shape data out of 776 subjects were obtained, which included 404 ($n_1$) typically developing children, 150 ($n_2$) diagnosed with ADHD-Combined, 8 ($n_3$) diagnosed with ADHD-Hyperactive-Impulsive,  and 85 ($n_4$) diagnosed with ADHD-Inattentive. Therefore, the data lie in the space $\Sigma_2^{50}$, which has a high dimension of $2\times 50-4=96$.  
 
 \begin{table}[htbp]
    \centering
    \caption{Demographic information about processed ADHD-200 CC shape dataset, including disease status, age, and gender.}
   \begin{tabular}{cccc}
    \hline Disease status & Num. & Range of age in years(mean) & Gender(female/male) \\
    \hline Typically Developing Children & 404 & $7.09-21.83(12.43)$ & $179 / 225$ \\
    ADHD-Combined & 150 & $7.17-20.15(10.96)$ & $39 / 111$ \\
    ADHD-Hyperactive/Impulsive & 8 & $9.22-20.89(14.69)$ & $1 / 7$ \\
    ADHD-Inattentive & 85 & $7.43-17.61(12.23)$ & $18 / 67$ \\
    All data & 647 & $7.09-21.83(12.09)$ & $237 / 410$ \\
    \hline
    \end{tabular}
    \label{tab-Demo}
\end{table}

As shown in the table \ref{tab-Demo}, we consider the classification problem with 4 different classes. We also divided the dataset into a $75$ percent training set and a $25$ percent test set and evaluated the classification accuracy in the test set compared to other learning methods. Since the sample size is unbalanced, the total number of some classes is too small, i.e., ADHD-Hyperactive case. We also considered the classification with two classes by combing those ADHD samples into one class shown in the right figure in Figure \ref{fig:landmark}.

\begin{figure*}[htbp]
    \centering
    \subfloat[Mean shapes of different classes]{
    \includegraphics[width=1.0\textwidth]{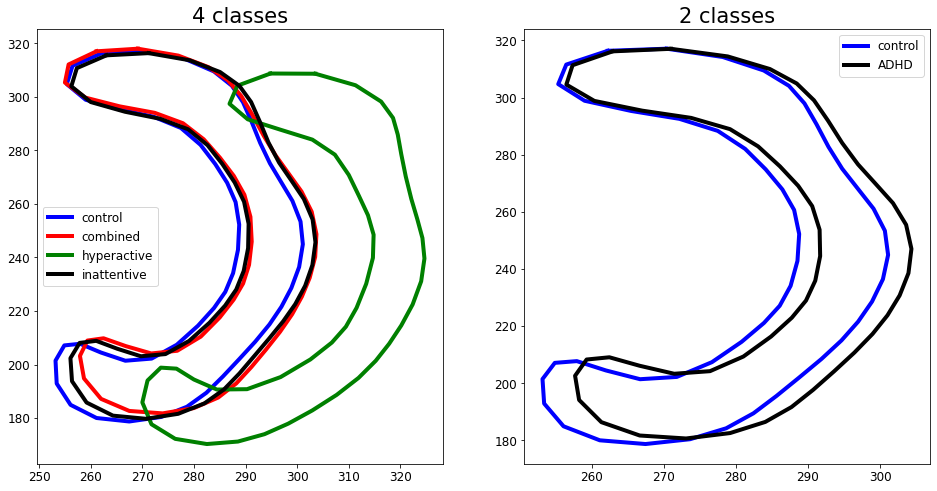}}
    
    \caption{CC shapes}
  \label{fig:landmark}
\end{figure*}

Similar to the sphere case, we select the k-nearest neighbors (kNN), the random forest (RF), the logistic regression (LR), and the support vector machine (SVM) with the radial basis function (RBF) kernel as competing estimators. The tuning parameters in each method are selected by evaluation on a validation data set whose size is $25\%$ of the training set. For all DNN models, we utilize the same network architecture of $5$ hidden layers with the numbers of width $(100,100,100,100,100)$. The DNN model is applied to the raw data, while the eDNN model is applied to the embedded data by Veronese-Whitney embedding. And the preshape data (normalized raw data) lying in the hyperspere $S^{100}$ is used for the tDNN model and iDNN model. In the iDNN model, we chose the north pole and south pole $(\pm 1,0,..,0)$ as base points and utilized the geometry of the hypersphere as before. In the tDNN model, we pick the $Frechet$ mean of the training set as the base point and transform all data in a batch to tangent vectors before feeding to the neural network. All models are trained with Adam optimizer. The competition results can be observed in Table 3. Our tDNN model achieves the best accuracy at $65.84\pm 3.10$ among 50 splits in the 2 classes case. Also, our iDNN model showed the best result of $63.55\pm 3.80$ in the 4 classes case.

\begin{table}[htbp]
    \centering
    \caption{The average accuracy on the test dataset is calculated over $50$ random splits. The $5$-layers network (with $100$ hidden nodes in each layer) is used for our DNN models in all experiments. Consequently, our tDNN model obtains the best accuracy in the 2 classes case while our iDNN model achieves the best accuracy in the 4 classes case. Furthermore, all our eDNN, tDNN and iDNN models outperform the classical DNN model, indicating the advantages of our frameworks. }
\begin{tabular}{lccc}
         & 4 Classes & 2 Classes\\
        \hline
        $\text{DNN}$  & $ 56.40  \pm 10.83  $ & $ 61.09 \pm 8.44 $ \\
        $\text{eDNN}$  & $   62.98\pm 3.91  $ & $  63.81\pm 3.72 $ \\
        $\text{tDNN}$ & $  63.20\pm 3.70 $ & $  \mathbf{65.84\pm 3.10} $ \\
        $\text{iDNN}$ & $ \mathbf{ 63.55\pm 3.80 }$ & $  65.42\pm 3.41 $ \\

        \hline
        $\text{kNN}$ & $ 57.62 \pm 3.37$ & $ 61.26\pm 3.84$\\
        $\text{LR}$& $ 61.35 \pm 3.54  $ &  $ 59.58 \pm 3.44$\\
        $\text{RF}$& $ 61.38 \pm 3.50  $ &  $ 63.20 \pm 3.13$\\
        $\text{SVM}$& $ 61.80 \pm 3.92  $ &  $ 64.89 \pm 3.64$\\
    \end{tabular}
    \label{tab:planar1}
\end{table}

\subsection{Symmetric semi-positive definite matrix (SPD)}

Covariance matrices are ubiquitous and attractive in machine learning applications due to their capacity to capture the structure inside the data. The main challenge is to take the particular geometry of the Riemannian manifold of symmetric positive definite ($\text{SPD}$) matrices into consideration. The space $\SPD(d)$ of all $d\times d$ positive definite matrices belongs to an important class of manifolds that possesses particular geometric structures, which should be taken into account for building the DNNs. \cite{Fletcher2007250} investigates its Riemannian structure and provides somewhat concrete forms of all its geometric quantities.  \cite{dryden2009non} studies different notions of means and averages in $\SPD(3)$ with respect to different distance metrics and considers applications to DTI data and covariance matrices.

Under the Riemannian framework of tensor computing \cite{pennec2006riemannian}, several metrics play an important role in machine learning on \text{SPD} matrices. Generally, the Riemannian distance $d(P_1, P_2)$ between two points $P_1$ and $P_2$ on the manifold is defined as the length of the geodesic $\gamma_{P_1 \to P_2}$, i.e., the shortest parameterized curve connecting them. In the $\text{SPD}$ manifold, the distance under the affine metric could be computed as \cite{pennec2006riemannian}:

\begin{align*}
d\left(P_{1}, P_{2}\right)=\frac{1}{2}\left\|\log \left(P_{1}^{-\frac{1}{2}} P_{2} P_{1}^{-\frac{1}{2}}\right)\right\|_{F}.
\end{align*}

Other important natural mappings to and from the manifold and its tangent bundle are the logarithmic mapping $Log_{P_0}$ and the exponential mapping $Exp_{P_0}$ at the point $P_0$. Under the affine metric, those two mappings are known in closed form:
\begin{align*}
\forall S \in \mathcal{T}_{P_{0}}, Exp_{P_{0}}(S)=P_{0}^{\frac{1}{2}} \exp \left(P_{0}^{-\frac{1}{2}} S P_{0}^{-\frac{1}{2}}\right) P_{0}^{\frac{1}{2}} \in \SPD(d) 
\end{align*}

\begin{align*}
\forall P \in \SPD(d),  Log_{P_{0}}(P)=P_{0}^{\frac{1}{2}} \log \left(P_{0}^{-\frac{1}{2}} P P_{0}^{-\frac{1}{2}}\right) P_{0}^{\frac{1}{2}} \in \mathcal{T}_{P_{0}},
\end{align*}
where $\mathcal{T}_{P_{0}}$ denotes the tangent space at $P_0$. Furthermore, we consider the log map on the matrix as the embedding $J$, mapping $\SPD(d)$ to $Sym(d)$, the space of the symmetric matrix. For example, let $P \in \SPD(d)$ with a spectral decomposition $P^{(l)}=U\Sigma U^{T}$, we have the log-map of $A$ as $\log(P) = U\log(\Sigma)U^{T}$ where $\log(\Sigma)$ denotes the diagonal matrix whose diagonal entries are the logarithms of the diagonal entries of $\Sigma$. Moreover, the embedding $J$ is a diffeomorphism, equivariant with respect to the actions of $GL(d,\R)$, the $d$ by $d$ general linear group. That is, for $H\in GL(d,\R)$, we have $\log(HPH^{T}) = H \log(P)H^{-1}$. 

In the context of deep networks on $\SPD$, we build up our model in terms of SPDNet introduced by \cite{huang2017riemannian}, which mimicked the classical neural networks with the stage of computing an invariant representation of the input data points and a second stage devoted to performing the final classification. The SPDNet exploited the geometry based on threefold layers:

The BiMap (bilinear transformation) layer, analogous to the usual dense layer; the induced dimension reduction eases the computational burden often found in learning algorithms on SPD data:
$$
X^{(l)}=W^{(l)^{T}} P^{(l-1)} W^{(l)} \text { with } W^{(l)} \text { semi-orthogonal. }
$$
The ReEig (rectified eigenvalues activation) layer, analogous to the ReLU activation, can also be seen as an Eigen-regularization, protecting the matrices from degeneracy:
$$
X^{(l)}=U^{(l)} \max \left(\Sigma^{(l)}, \epsilon I_{n}\right) U^{(l)^{T}}, \text { with } P^{(l)}=U^{(l)} \Sigma^{(l)} U^{(l)^{T}.}
$$
The LogEig (log eigenvalues Euclidean projection) layer:
$X^{(l)}=\operatorname{vec}\left(U^{(l)} \log \left(\Sigma^{(l)}\right) U^{(l)^{T}}\right)$, with again $U^{(l)}$ the eigenspace of $P^{(l)}$.

Under our framework, the SPDNet is both an eDNN and a tDNN model. The LogEig layer applies the logarithmic mapping $\log_{I}(P) = \operatorname{vec}\left(U^{(l)} \log \left(\Sigma^{(l)}\right) U^{(l)^{T}}\right)$, which is identical to the transformation in the LogEig layer. Thus, SPDNet can also be viewed as a tDNN model. In our experiments, we only consider tDNN models as one tangent space from the base point is sufficient to cover the entire manifold. Our eDNN models on $\SPD(p)$ consist of 3 BiMap layers, 3 ReEig layers, one LogEig layer (for embedding), and a 5-layer DNN with 100 hidden nodes per layer. In tDNN models, we replace the LogEig layer with the intrinsic logarithmic mapping under different metrics.

In our experiments, we evaluate the performance of tDNN and eDNN models on the AFEW and HDM05 datasets using the same setup and protocol as in \cite{huang2017riemannian}. The AFEW dataset \cite{dhall2011static} includes 600 video clips with per-frame annotations of valence and arousal levels and 68 facial landmarks, depicting 7 classes of emotions. The HDM05 dataset \cite{muller2007mocap} contains over three hours of motion capture data in C3D and ASF/AMC formats, covering more than 70 motion classes across multiple actors. We divide the data into a 75-25 percent training-test split, with 10 repetitions, and use the validation set (25 percent of training data) to tune hyperparameters. We implement tDNN models on both affine metrics and log-Euclidean metrics, using the Frechet mean of the batch as the base point. As shown in Table \ref{tab:spd}, our tDNN model under the Log-Euclidean metric achieves the best results on both datasets, with a 35.85 $\pm$ 1.49 accuracy on the AFEW dataset and 62.59 $\pm$ 1.35 accuracy on the HDM05 dataset.

\begin{table}[htbp]
    \centering
    \caption{The accuracy of the test set was reported. We follow the setup and protocols in  \cite{huang2017riemannian} and our tDNN models outperform the eDNN (SPDNet) under both log and affine metrics.  }
\begin{tabular}{lccc}
\hline Data & AFEW & HDM05    \\
\hline$(n, d)$ & $(2135,400^2)$ &$(2086,93^2)$   \\
\hline \hline eDNN(SPDNet) & $34.23 \pm 1.44 $ & $61.35 \pm 1.12$  \\
tDNN-Log & $\mathbf{35.85 \pm 1.49}$ & $\mathbf{62.59 \pm 1.35}$    \\
tDNN-Affine &$35.31 \pm 1.68$  & $62.23 \pm 1.43$  \\
\hline
\end{tabular}
    \label{tab:spd}
\end{table}

\section{Discussion}
In this work, we develop intrinsic and extrinsic deep neural network architectures on manifolds and  characterize their theoretical properties in terms of approximation error and statistical error of the ERM based estimator. The neural networks explore the underlying geometry of the manifolds for learning and inference. Future work will be focused on developing convolutional neural networks in manifolds for image classifications of manifold-values images, which have abundant  applications in medical imaging and computer vision.  

\section*{Acknowledgments}
We would like to thank Dong Quan Nguyen, Steve Rosenberg, and Bayan Saparbayeva for the very helpful discussions. We acknowledge the generous support of NSF grants DMS CAREER 1654579 and DMS 2113642. The second author was supported by INHA UNIVERSITY Research Grant.

%Bibliography
\bibliographystyle{abbrv}
\bibliography{references}

\end{document}